\documentclass[twoside,11pt]{article}

\usepackage{jmlr2e}

\usepackage{amsmath} 
\usepackage{amssymb} 
\usepackage[scr]{rsfso} 
\usepackage{graphicx}
\usepackage{subfig}
\usepackage{rotating}

\usepackage{algorithm}
\usepackage{algpseudocode}

\jmlrheading{Under Review}{2020}{}{7/2020}{X/2020}{Zest AI}

\ShortHeadings{Path choice in game-theoretic attribution algorithms}{Ward, Kamkar, and Budzik}
\firstpageno{1}

\begin{document}

\title{An exploration of the influence of path choice in game-theoretic attribution algorithms}

\author{\name Geoff M. Ward \email gmw@zest.ai \\
       \name Sean J. Kamkar \email sjk@zest.ai \\
       \name Jay Budzik \email j@zest.ai \\
       \addr Zest AI \\
			  3900 W. Alameda Ave \\
				Burbank, CA 91505
			}

\editor{TBD}

\maketitle

\begin{abstract}

  We compare machine learning explainability methods based on the theory of atomic
  \citep{shapley} and infinitesimal \citep{aumann-shapley} games, in a theoretical and experimental investigation 
	into how the model and choice of integration path can influence the resulting feature attributions.  To gain insight into differences in attributions 
	resulting from interventional Shapley values \citep{sundararajan2019many, janzing2019feature, chen2019explaining} and 
	Generalized Integrated Gradients (GIG) \citep{merrill2019generalized} we note interventional Shapley is equivalent to a 
	multi-path integration along $n!$ paths where $n$ is the number of model input features.  Applying Stoke's theorem we show that the path
	symmetry of these two methods results in the same attributions when the model is composed of a sum of separable functions of 
	individual features and a sum of two-feature products.  We then perform a series of experiments with varying degrees of data missingness 
	to demonstrate how interventional Shapley's multi-path approach can yield less consistent attributions than the single straight-line path 
	of Aumann-Shapley. We argue this is because the multiple paths employed by interventional Shapley extend away from the training data manifold and are therefore more likely to pass 
	through regions where the model has little support.  In the absence of a more meaningful path choice, we therefore advocate the 
	straight-line path since it will almost always pass closer to the data manifold.  Among straight-line path attribution algorithms 
	GIG is uniquely robust since it will still yield Shapley values for atomic games modeled by decision trees.

\end{abstract}


\section{Introduction}

The notion of additive attribution was first introduced into game theory by \cite{shapley} and has been a persistent theme in research literature ever since, due to its broad utility and theoretical underpinning.  Most recently it has found its way into machine learning model attribution algorithms \citep{lundberg2017consistent,lundberg2017unified,lundberg2018consistent,sundararajan2017axiomatic,merrill2019generalized}.  The primary appeal of Shapley game-theory-based attribution techniques can be seen in the many axioms that they attempt to satisfy:
\begin{itemize}
  \item  \textbf{\emph{Efficiency}}: The sum of attributions should equal the change in score,
\begin{equation}
f(\vec{x}_b)-f(\vec{x}_a) = \sum_{i=1}^{n}\phi_i.
\end{equation}
Here $f(\vec{x})$ is a model taking input feature vector $\vec{x}$, $a$ and $b$ are two points in feature space, and $\phi_i$ is the attribution assigned to the $i^{th}$ feature in distinguishing the model output difference between the two points.  Also referred to as \emph{local accuracy} in \cite{lundberg2017consistent}.
\item \textbf{\emph{Sensitivity}}: If the model output changes under the change of one feature then that particular feature must receive some level of attribution.  Also referred to as \emph{consistency} in \cite{lundberg2017consistent}.
\item \textbf{\emph{Invariance}}:  Two models that score every point in space the same way should also yield the same feature attributions.
\item \textbf{\emph{Linearity}}: A model composed of a linear combination of other models must yield attribution corresponding to linear combination of the individual model attributions.
\item \textbf{\emph{Dummy}}: A feature that the model does not depend on should receive zero attribution.  This is grouped under the \emph{sensitivity} axiom in \cite{sundararajan2017axiomatic} and also referred to as \emph{insensitivity}.
\item \textbf{\emph{Symmetry}}:  Feature ordering should not influence attributions in otherwise identically scoring models.  For example, if $f(x_1,x_2) \equiv g(x_2,x_1)$ then any sound attribution algorithm applied to $f(x_1,x_2)$ and $g(x_2,x_1)$ should give $x_1$ and $x_2$ the same attribution between points $a$ and $b$.
\end{itemize}
While other desirable properties exist, these form the core theoretical justification for the choice of Shapley-based techniques for model explainability.  The addition of further axioms gives various Shapley-based attribution algorithms claims of uniqueness \citep{sundararajan2019many}.

The failure of the conditional expectation version of Shapley values \citep{lundberg2017consistent} to meet the \emph{dummy} and \emph{linearity} axioms was recently highlighted in \cite{merrick2019explanation}, \cite{janzing2019feature}, and \cite{sundararajan2019many}.  Switching to unconditional expectations, as provided by Pearl's $do$ operator \citep{pearl2009causality}, resolves the issue.  Borrowing from causal inference, this yields interventional Shapley values (first called Baseline Shapley (BS) in \cite{sundararajan2019many}), which is a generalization of Shapley-Shubik \citep{friedman1999three}.

While Shapley value model attribution techniques are inspired by atomic, cooperative game theory, a parallel body of model attribution techniques exists that is based on infinitesimal, cooperative game theory \citep{aumann-shapley}.  Central to such attribution techniques is the idea of gradient path integrals \citep{sundararajan2017axiomatic,FriedmanPaths,friedman1999three}.  A straight-line path between points represents the simplest choice to satisfy the symmetry-preserving axiom, however, there are many other possible path integral attribution methods, including multi-path ones, that satisfy varying combinations of the desired axioms of additive attribution methods.  The core of all of these algorithms is the need to calculate the gradient of the model prediction function.  The dependence on the gradient of the model prediction function means that all path integration attribution techniques inherently satisfy the sensitivity axiom.  However, the requirement for a model prediction gradient left such methods inapplicable to decision tree models which are not everywhere differentiable.

This shortcoming was solved by the introduction of Generalized Integrated Gradients (GIG) \citep{merrill2019generalized}, which extends Integrated Gradients (IG) \citep{sundararajan2017axiomatic} to provide attribution for models with discontinuities that are hyperplanes orthogonal to features, such as decision trees.  GIG provides an axiomatic method for attribution of mixed typed games involving atomic and non-atomic pieces.  This places GIG in a unique position amongst attribution techniques, particularly with regard to decision trees where atomic game theory-based attribution has been the dominant mode of attribution. 

In reading through research papers on game theory-based attribution methods it becomes clear that the uniqueness claims of each technique generally hinge on a slight variation in the axioms satisfied.  For most machine learning applications that claim to uniqueness is irrelevant to the problem at hand.  This paper demonstrates more clearly the significance of the path \emph{choice} in path-integral attribution methods.  After introducing the GIG decision tree algorithm we then express interventional Shapley as a multi-path integral attribution method so we can compare the multiple interventional Shapley paths with GIG, which uses a single path.  This makes more clear that it is the path \emph{choice}, whose properties are often summarized in axiom form, that sets various game theory-based attribution methods subtly apart from one another.  We show in the studies below that the \emph{choice} of path results in differences in attributions that are best interpreted in light of that choice and can be sensitive to model uncertainty.

\section{Generalized Integrated Gradients}
We begin with a concise description of Generalized Integrated Gradients (GIG).  A full treatment of the mathematical theory behind GIG can be found in \cite{merrill2019generalized}.  GIG extends Integrated Gradients (IG) to provide attribution for models that involve hyperplane discontinuities that are orthogonal to the input feature set, such as those produced by decision trees.
\subsection{GIG basics}
The path choice for GIG is a straight-line between two points.  We refer to the starting point, $x_a$, as the reference point and the ending point, $x_b$, as the explanation point.  In other texts the reference point is called the baseline \citep{sundararajan2019many}.  Between discontinuities, in continuous regions, the attribution along the path is simply
\begin{equation}\label{ig_eq}
 \phi_i = \Delta x_{i_{ab}}\int_{\alpha=0}^{\alpha=1}\frac{\partial f}{\partial x_i}\biggr\rvert_{\vec{x}=\vec{x}_a+\alpha \Delta x_{i_{ab}}} d\alpha,
\end{equation}
where $\Delta x_{i_{ab}}=x_{i_b}-x_{i_a}$. In extending IG, the primary contribution that GIG provides is answering the question of how to attribute when the path intersects a discontinuity.  It should not be entirely surprising that the attribution at discontinuities comes in the form of Shapley values, not defined by feature missingness, but rather orthants (quadrants in two-dimensions):
\begin{equation}
	\phi_i^{disc} = \sum_{S \subseteq N \setminus \{ i \}} \frac{|S|! (N - |S| - 1)!}{N!} (f(\vec{x}_{S \cup \{ i \}}) - f(\vec{x}_{S}))
\end{equation}
where $S$ is the set of orthants surrounding the discontinuity limited to those with $x_i^-$ (before the discontinuity $i^{th}$ feature-orthogonal hyperplane) and $S \cup \{ i \}$ refers to all of the surrounding orthants limited to those with $x_i^+$ (after the discontinuity $i^{th}$ feature-orthogonal hyperplane).  The permutation form of Shapley values,
\begin{equation}
\phi_i^{disc} = \frac{1}{n!} \sum_{j=1}^{n!} f(\vec{x}_{S_{P_j} \cup \{ i \}}) - f(\vec{x}_{S_{P_j}})
\end{equation}
where $S_{P_j}^i$ refers to the subset of features up to the $i^{th}$ feature in the $j^{th}$ feature ordering $P_j$, suggests an interesting alternate interpretation as to how attributions are made effectively by an average over $n!$ paths defined by feature-orthogonal segments infinitesimally close to the discontinuity.  For the simplest case of the path intersecting a single-feature orthogonal discontinuity we note that all of the attribution simply goes to that single feature and is merely the change in value across the discontinuity, which matches intuition.

In the following subsection we describe the practicalities of the GIG decision tree algorithm in some detail.  The subsection after that explains how we calculate and measure the convergence of attribution expectations over a reference population to guarantee convergence while avoiding unnecessary computation.

\subsection{The GIG Decision Tree Algorithm}
Efficiently finding the intersection between decision tree feature splits and the integration path presents a difficult challenge.  It is anticipated that the sorting of feature splits inherently provided by a decision tree can be exploited in a successful algorithm. Furthermore, it is expected that it is possible to simultaneously extract the prediction value between the splits during the identification of intersections.  Algorithm \ref{FindSplits} achieves both of these desired qualities for a single decision tree.

The algorithm is presented in a recursive form.  The path is parametrized by $\alpha$ with the line segment extending from $\alpha_{min}=0$ to $\alpha_{max}=1$.  Starting from the first node, $node=0$, the tree is traversed until a splitting feature is found.  Upon finding a splitting feature the procedure is called recursively on the two new path segments defined by the split.  The segment with smaller values of $\alpha$ is recursively traversed first.  Eventually the region defined by the split with the smallest possible $\alpha$ is reached.  At that point a leaf of the tree has been reached.  The $\alpha$ value of that smallest split is then recorded as well as the leaf value and the feature that caused the split.  The process is complete when all recursions have ended at a leaf and the values have been recorded.

The result of Algorithm \ref{FindSplits} is a sorted listing of $\alpha$'s, the features that caused the split, and the leaf values.  This process must be repeated for each tree in an ensemble.  For ensembles of decision trees whose leaf values are in margin space the alphas from each tree must then be merged across trees to get the actual scores in a constant valued region defined by splits.  The best practice for achieving this merge of sorted values is through heaps, which provide an $O(n_{\alpha}log_2(n_{trees}))$ time merge.

We note that this version of the GIG decision tree algorithm does not deal with path intersections with corners (discontinuities involving more than two orthants).  Extending the algorithm to do so merely requires traversing both branches when a repeated alpha split is identified.  Assigning attribution then requires dealing with the problem of an exponentially growing number of orthants, $2^{n_{orth}}$, which can be dealt with by sampling \citep{kononenko2010efficient,aas2019explaining}.  In the worst case scenario the algorithm visits each leaf of the tree and therefore has an asymptotic upper bound $O(n_{tree}MAX(depth,n_{leaf}/2))$.  If merging splits is required then an additional worst case $\sum_{i=1}^{n_{tree}}n_{leaf_i} log_2(n_{tree})$ merge operation results.

Lastly, we mention that while this algorithm assigns attribution from one point to another, in practice an expectation is taken over many reference points.  An important question to answer when calculating an expected value based on a population is whether that value has converged.  Because expectation of interest is a simple mean of separate path integrals it becomes straightforward to measure their convergence and uncertainty \citep{castro2009polynomial,maleki2013bounding,merrick2019explanation}.

\begin{algorithm}
	\caption{Find values of $\alpha$ for discontinuities produced by a single decision tree}
	\label{FindSplits}
	\begin{algorithmic}
	\State $tree \gets feature$, $threshold$, $value$, $left$, $right$
	\State $path \gets x_{start},x_{end}$
	\State $\alpha_{splits} \gets \lbrack \quad \rbrack$
	\State $feature_{splits} \gets \lbrack \quad \rbrack$
	\State $leaf_{splits} \gets \lbrack \quad \rbrack$
	\State $Initial$ $node \gets 0$	
	\State $Initial$ $\alpha_{min},\alpha_{max} \gets 0,1$		
	\Procedure{FindPathTreeSplits}{$\alpha_{min}$,$\alpha_{max}$,$node$}
		\If{$left \lbrack node \rbrack == right \lbrack node \rbrack$}[Leaf node was reached]
			\State $leaf_{splits}.append(value \lbrack node\rbrack)$
			\State $feature_{splits}.append(feature \lbrack node \rbrack)$
			\State $\alpha_{splits}.append(\alpha_{max})$
		\Else
			\State $f_{\alpha_{min}} = x_{start}\lbrack feature\lbrack node\rbrack\rbrack \alpha_{min}+ x_{end}\lbrack feature\lbrack node\rbrack\rbrack(1-\alpha_{min})$
			\State $f_{\alpha_{max}} = x_{start}\lbrack feature\lbrack node\rbrack\rbrack \alpha_{max}+ x_{end}\lbrack feature\lbrack node\rbrack\rbrack(1-\alpha_{max})$
			\If{$f_{\alpha_{min}} < threshold\lbrack node\rbrack$}
				\State $node_{\alpha_{min}} = left\lbrack node\rbrack$
			\Else
				\State $node_{\alpha_{min}} = right\lbrack node\rbrack$
			\EndIf
			\If{$f_{\alpha_{max}} < threshold\lbrack node\rbrack$}
				\State $node_{\alpha_{max}} = left\lbrack node\rbrack$
			\Else
				\State $node_{\alpha_{max}} = right\lbrack node\rbrack$
			\EndIf
			\If{$node_{\alpha_{min}} == node_{\alpha_{max}}$} [ Feature split did not intersect the path]
				\State FindPathTreeSplits$(\alpha_{min},\alpha_{max},node_{\alpha_{min}})$
			\Else
				\State $f_{mid}=threshold\lbrack node\rbrack$
				\State $f_{start}=x_{start}\lbrack feature\lbrack node\rbrack \rbrack$		
				\State $f_{end}=x_{end}\lbrack feature\lbrack node\rbrack \rbrack$							
				\State $\alpha_{mid} = (f_{mid}-f_{start})/(f_{end}-f_{start})$
				\State FindPathTreeSplits$(\alpha_{min},\alpha_{mid},node_{\alpha_{min}})$
				\State FindPathTreeSplits$(\alpha_{mid},\alpha_{max},node_{\alpha_{min}})$
			\EndIf
		\EndIf
	\EndProcedure
\end{algorithmic}
\end{algorithm}

\section{Attribution from path integration}
The previous section described the extension of the straight-line path integration technique of Aumann-Shapley to deal with a subset of discontinuous scoring functions, including decision trees.  In the following subsections we explore path choice in more depth.  To begin we expose the fact that the interventional Shapley values attribution algorithm is actually an $n!$ multi-path integration attribution algorithm.  While this has already been mostly noted by \cite{sundararajan2019many}, we emphasize it more directly with respect to actual gradient path integration. Even though gradients are irrelevant for truly atomic games, we find it still provides some insights into differences in attribution between the methods.

\subsection{Interventional Shapley values from path integrals}
For atomic games \cite{shapley} introduced the additive attribution
\begin{equation}
	\phi_i = \sum_{S \subseteq N \setminus \{ i \}} \frac{|S|! (N - |S| - 1)!}{N!} (\nu(f,\vec{x}_{S \cup \{ i \}}) - \nu(f,\vec{x}_{S}))
\end{equation}
where $\nu$ is the \emph{lift} that specifies how $f$ should be evaluated under feature missingess and $S$ are subsets of features of varying degrees of feature missingness.  To meet the requirement of the \emph{dummy} axiom it was noted by \cite{janzing2019feature} and \cite{sundararajan2019many} that the use of unconditional expectations were necessary when evaluating the scoring function under feature missingness.  This yields an expectation that is merely an average of the attribution relative to a single reference point considered at a time,
\begin{equation}
	\phi_i(\vec{x}) = \frac{1}{n_{ref}}\sum_{j=1}^{n_{ref}}\hat{\phi}_i(\vec{x},\vec{x}_{{ref}_j}),
\end{equation}
where the attribution relative to a single reference is defined by
\begin{equation}
	\hat{\phi}_i(\vec{x},\vec{x}_{ref}) = \sum_{S \subseteq N \setminus \{ i \}} \frac{|S|! (N - |S| - 1)!}{N!} (f(\vec{x}_{S \cup \{ i \}},\vec{x}_{\overline{S \cup \{ i \}}}^{ref}) - f(\vec{x}_{S},\vec{x}_{\overline{S}}^{ref})).
\end{equation}
Here the over bar is used to denote the subset of missing features.  We see that evaluation under missingness for interventional Shapley attribution simply means replace the feature value with that of the reference sample's value.

To show that this is equivalent to a multi-path integration attribution method, recall that the first fundamental theorem of vector calculus states that 
\begin{equation}
f(\vec{x}_b)-f({\vec{x}_a}) = \int_{\vec{x}_a}^{\vec{x}_b}\vec{\nabla}f(\vec{x}) \cdot d\vec{x}
\end{equation}
for any given path between $\vec{x}_a$ and $\vec{x}_b$.  Path integral attribution techniques, as in equation \ref{ig_eq}, simply decompose this equation by gradient component when assigning credit.  If the path of integration is strictly orthogonal to a feature then only that feature receives any attribution and in accordance with the first fundamental theorem of vector calculus the attribution is equivalent to simply the change in the function.  All of the integration work needed to render credit assignment to different features becomes unnecessary for such orthogonal paths.

Now, returning to the permutation form of Shapley's additive attribution,
\begin{equation}
\phi_i = \frac{1}{n!} \sum_{j=1}^{n!} \nu(f,\vec{x}_{S_{P_j}^i \cup {i}})-\nu(f,\vec{x}_{S_{P_j}^i}),
\end{equation}
where $S_{P_j}^i$ refers to the subset of features up to the $i^{th}$ feature in the $j^{th}$ feature ordering $P_j$,  recalling from GIG that this can be thought of in terms of path integration, we see that interventional Shapley values are just a gradient path integration along the $n!$ paths that traverse in segments orthogonal to each of the input features corresponding to a feature ordering permutation.  This summation represents bringing together all of the subsections of the $n!$ paths that are orthogonal to the $i^{th}$ feature.  Any need to actually perform path integrals is avoided despite interventional Shapley values being equivalent to a path integration approach.  While most atomic games don't even have a meaningful gradient between states, this interpretation of interventional Shapley values is still highly informative in understanding what they actually represent in relation to other path integration techniques.  In two-dimensions the two interventional Shapley paths are depicted in the rightmost panel of Figure \ref{fig_shappaths}. 

While there are many possible choices of paths, the family corresponding to interventional Shapley values have the additionally desirable property that a nonlinear transformation of any feature,
\begin{equation}
x_i' = f(x_i)
\end{equation}
does not alter the resulting interventional Shapley value.  Alternatively, any nonlinear transformation of a feature effectively modifies the path and therefore the attribution in other path integration attribution techniques.  This includes the path of \emph{serial cost sharing}  (see middle panel of Figure \ref{fig_shappaths}), which presents an alternative single path integration method based on connected straight-line segments that, unlike Aumann-Shapley, does not exhibit feature scale invariance.  In fact there are an infinite number of multi-path methods satisfying the core attribution axioms.  Notably, it is always possible to choose any path so long as $n!$ symmetry complementing paths are also considered in equal weight.

\begin{figure}
\centering
\includegraphics[width=.3\textwidth]{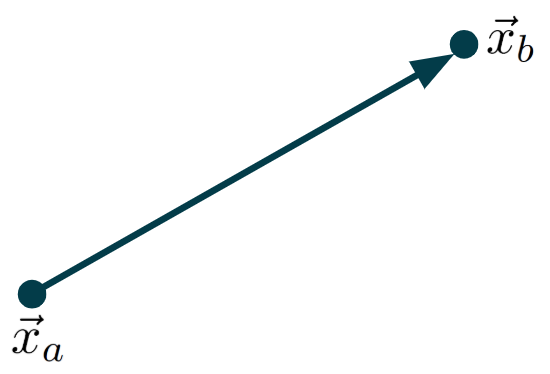}\hfill
\includegraphics[width=.3\textwidth]{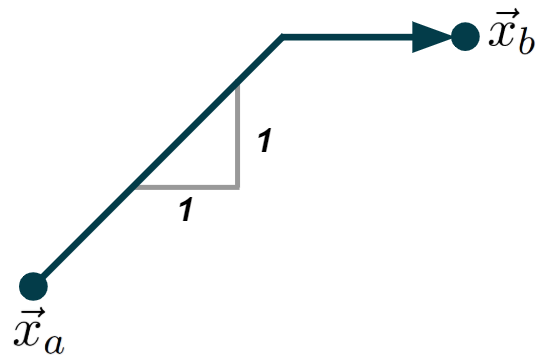}\hfill
\includegraphics[width=.3\textwidth]{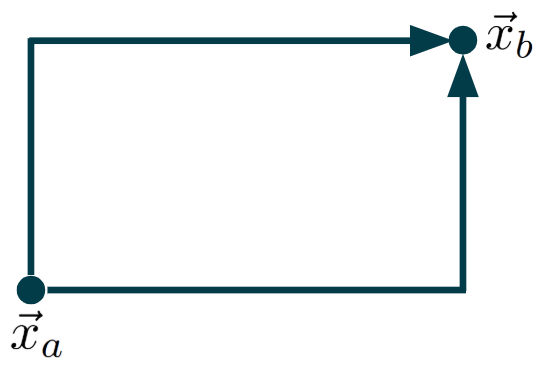}
\caption{The various paths associated with Shapley model attribtuion techniques: \emph{(Left)} The single straight path of \emph{Aumann-Shapley}. \emph{(Center)} The single path of \emph{serial cost sharing}, advancing each feature at the same rate until their final value is reached. \emph{(Right)} The multiple paths of \emph{interventional Shapley}.  In two-dimensions there are just two, but in the general case of $n$ model input features there are $n!$.}
\label{fig_shappaths}
\end{figure}

\subsection{From Aumann-Shapley to Shapley values}
What connection, if any, can be made between Shapley and Aumann-Shapley attribution values?  In many situations taking the expectation over a set of references for gradient integration techniques seems to yield something that approaches Shapley values.  A trivial example of where this holds true is for linear models. However, there are many other situations where this also seems to be the case, as was seen in the two ovals example in \cite{merrill2019generalized}.  The reason is not that difficult to parse out: the Shapley value equation involves model evaluation differencing in a way that behaves like an averaging of gradients over the domain of possible values under missingness.  This was highlighted in the previous subsection which showed that interventional Shapley values can be thought to come from an equivalent multi-path gradient integration approach.  Generally speaking though, one would think that the closer the scoring function is to linear and the more the reference population used is uniformly and independently distributed then path integration techniques will be close to Shapley values when an expectation is taken over the reference population.  In the following paragraph we demonstrate that for some scoring functions there is an equivalence for certain path integration methods that extends beyond linear functions.

In three-dimensions, Stoke's Theorem gives some insight into when the path is completely irrelevant.  From vector calculus, recall,
\begin{equation}
\int_C \vec{F}\cdot d\vec{r} = \iint_S \vec{\nabla} \times \vec{F} \cdot d\vec{S},
\end{equation}
where $C$ is some contour, $\vec{F}$ is a vector field, and $S$ is a surface with $C$ as its edge.  If we take
\begin{equation}
\vec{F} = (\frac{\partial f}{\partial x_1},0,0)
\end{equation}
and consider the two paths seen in Figure \ref{fig_stokes}, then Stoke's theorem gives
\begin{equation}
\int_{C_2} \vec{F}\cdot d\vec{r}-\int_{C_1} \vec{F}\cdot d\vec{r} = \iint_S \vec{\nabla} \times \vec{F} \cdot d\vec{S},
\end{equation}
where $S$ is the surface with edge defined by $C_1+C_2$.  The curl of the vector field is
\begin{equation}
\vec{\nabla} \times \vec{F} = (0,\frac{\partial^2 f}{\partial x_1 \partial x_3},-\frac{\partial^2 f}{\partial x_1 \partial x_2})
\end{equation}
we see then that, at least in up to three-dimensions, 
\begin{equation}
\frac{\partial^2 f}{\partial x_i \partial x_j} \equiv 0,
\hspace{2pt} \forall \hspace{2pt} i,j, \hspace{2pt} i \ne j \Rightarrow 
\int_{C_1} \vec{F} \cdot d \vec{r} = 
\int_{C_2} \vec{F} \cdot d \vec{r}
\end{equation}
implies that the path taken is irrelevant to assigning attribution for twice differentiable continuous scoring functions satisfying the differential requirement.  Even though Stoke's theorem doesn't apply to discontinuous functions, a more general conclusion can be made:
\begin{theorem}
For any function of the form
\begin{equation}
f(\vec{x}) = const+\sum_{i=1}^n f_i(x_i),
\end{equation}
the only possible attribution for a point $\vec{x}_b$ relative to a point $\vec{x}_a$ that satisfies the \emph{efficiency}, \emph{sensitivity}, \emph{invariance}, \emph{linearity}, \emph{dummy}, and \emph{symmetry} axioms is the path independent attribution
\begin{equation}
\phi_i = f_i(x_{i_b})-f_i(x_{i_a})
\end{equation}
\end{theorem}
\begin{proof} 
Emphasizing that any such function is a linear combination of functions, in order to satisfy \emph{linearity}, each function must independently satisfy all of the desired axioms.  Because each function only depends on a single input feature then the only way it can satisfy the remaining axioms is to attribute to that single feature the simple difference in the resulting score it produces.
\end{proof}
From this theorem and Stoke's theorem we can also see that models of the form
\begin{equation}
f(\vec{x}) = const+\sum_{i=1}^n f_i(x_i)+\sum_{j=1}^n\sum_{i=1}^n m_{ij} x_i x_j
\end{equation}
will yield a simple shift in attributions since $\frac{\partial^2 f}{\partial x_i \partial x_j}=m_{ij}$.  The shift will be linearly proportional to the area swept out between the two paths.  If we consider interventional Shapley's two paths in conjunction with Aumann-Shapley's straight path in two-dimensions (see center Figure \ref{fig_stokes}), Stoke's theorem for such functions gives
\begin{equation}
\int_{C_2} \vec{F}\cdot d\vec{r}-\int_{C_3} \vec{F}\cdot d\vec{r} = -\iint_{S_2} m_{ij} dS
\end{equation}
and
\begin{equation}
\int_{C_1} \vec{F}\cdot d\vec{r}-\int_{C_3} \vec{F}\cdot d\vec{r} = \iint_{S_1} m_{ij} dS.
\end{equation}
Adding the two together, realizing that $S_1$ and $S_2$ sweep out the same area, we see that 
\begin{equation}
\frac{1}{2}\int_{C_1} \vec{F}\cdot d\vec{r}+\frac{1}{2}\int_{C_2}\vec{F}\cdot d\vec{r} = \int_{C_3} \vec{F}\cdot d\vec{r}
\end{equation}
which indicates that the straight-line path yields the same attribution as interventional Shapley values.  In fact any two paths falling entirely in the same plane that sweep out the same area and are on opposite sides of the straight-line path will lead to interventional Shapley values.  The surface swept out in higher dimensions becomes more complex along with the $n!$ paths, but it is not hard prove that the result still holds true:

\begin{theorem}
Given an explanation point and a reference point in $n$-dimensions, all functions of the form
\begin{equation}
f(\vec{x}) = const+\sum_{i=1}^n f_i(x_i)+\sum_{j=1}^n\sum_{i=1}^n m_{ij} x_i x_j
\end{equation}
yield the same attributions from Aumann-Shapley and interventional Shapley.
\end{theorem}

\begin{proof}  
To begin we draw on the previous theorem that proved that attribution for $const+\sum_{i=1}^n f_i(x_i)$ is completely path independent and therefore Aumann-Shapley and interventional Shapley yield the same attributions for that portion.  Our attention is then only needed on $\sum_{j=1}^n\sum_{i=1}^n m_{ij} x_i x_j$.  Because it is twice differentiable we can leverage Stoke's theorem to fill in the missing piece of the proof.

We note that in higher dimensions any differential surface element will have a Stoke's curl flux through it that is merely a linear combination of the $m_{ij}$:
\begin{equation}
\sum_{i=1}^n\sum_{j=1}^n \alpha_{ij}m_{ij} \vert dS \vert,
\end{equation}
where the coefficients $\alpha_{ij}$ depend only on the orientation of the surface element.  To prove equivalence it is only necessary to show that the surfaces swept out by path pairs from the $n!$ interventional Shapley paths yield cancellation due to the sign of the Stoke's curl flux being flipped when the orientation of the surface is flipped.  Any further understanding of Stoke's theorem in higher dimensions is not needed to prove the theorem.

Next, recall that each of the $n!$ paths associated with interventional Shapley can be thought of in terms of orthogonal path segments defined by a feature ordering.  For instance, an ordering might be $x_5, x_3, \cdots, x_8, x_2, \cdots, x_1, x_7$, which specifies the first orthogonal segment changes feature $x_5$ from its reference value to explanation point value, followed by similar orthogonal segment for $x_3$, and so on until the last orthogonal segment is completed by changing $x_7$.  In applying Stoke's theorem in higher dimensions it becomes possible to pick any surface that has the path of interest as its edge.  First, take the edge of interest as a single interventional Shapley ordering path and the straight-line Aumann-Shapley path.  For the purposes of proving the theorem, consider the surface to be defined by a strip of triangles.  The triangles in the strip alternate between having two corners from the interventional Shapley ordering path corners and one corner from a point along the center line and two corners from points along the center line path and one corner from an interventional Shapley ordering path corner.  The points along the straight center line are defined by lines passing through the corners of the interventional Shapley ordering path that are perpendicular to the center line.  The surface orientation of each of these triangles is defined by the right hand rule and the interventional Shapley path feature ordering.

To complete the proof we note that for every interventional Shapley feature ordering path there is an inverse feature ordering path.  For instance, $x_7, x_1, \cdots, x_2, x_8, \cdots, x_3, x_5$ is the inverse feature ordering for $x_5, x_3, \cdots, x_8, x_2, \cdots, x_1, x_7$.  In simply switching the feature ordering, we see that the two resulting triangle strips contain all of the same shaped triangles except in the exact opposite ordering and therefore opposite orientation (see Figure \ref{fig_stokes} right), thus resulting in cancellation of their Stoke's curl flux when added together.  For $n>1$ we see that all interventional Shapley feature ordering paths have an inverse feature ordering path.  Therefore Aumann-Shapley and interventional Shapley attributions for such functions will be the same for any dimensionality.  The case of $n=1$ yields a trivial special case where interventional Shapley has a single path that is identical to the path choice for Aumann-Shapley.
\end{proof}

\begin{figure}
\centering
\includegraphics[height=.23\textwidth]{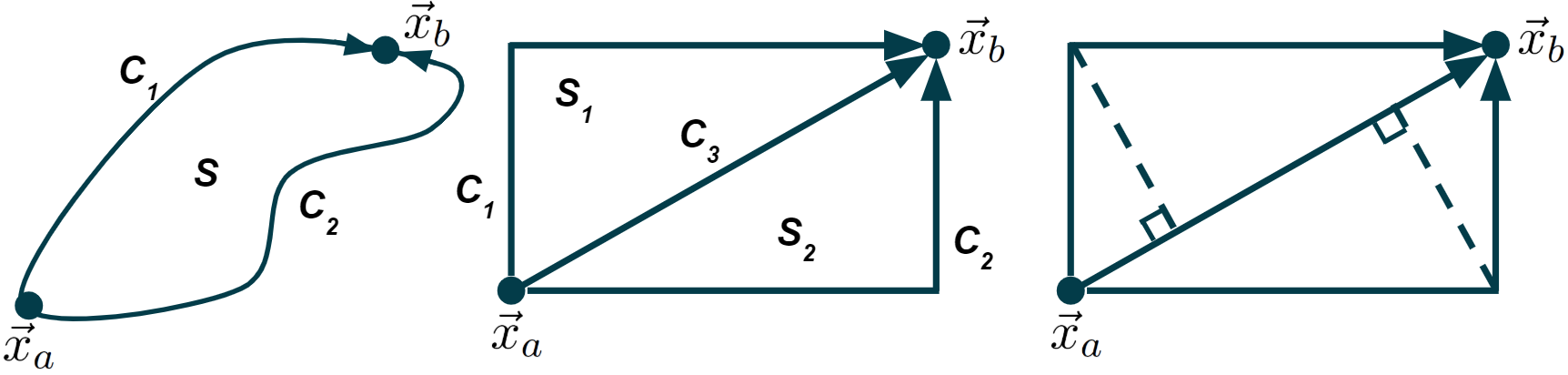}
\caption{(\emph{Left}) Two potential path's for attribution.  In three-dimensions, for twice differentiable functions, the difference in resulting attribution can be recovered from Stoke's theorem. (\emph{Center}) Interventional Shapley and Aumann-Shapley paths yield a symmetry in the surface they sweep out with respect to the straight-line path that yields gradient path integration equivalence for nonlinear scoring functions with constant mixed partial derivatives. (\emph{Right})  The strips of triangles swept out by the only two paths in two dimensions, each being the inverse ordering path of the other.}
\label{fig_stokes}
\end{figure}

From this simple application of Stoke's theorem we see that for single path gradient integration techniques the resulting expected attributions begin to clearly diverge from interventional Shapley values when non-linearity is not feature-separable.  When the reference population highlights regions with such behavior then the values can further diverge, indicating that for many applications a uniformly and independently distributed set of reference points should also yield attributions that more closely match Shapley values.  The symmetry of the straight-line path and interventional Shapley paths has extended equivalence for functions that are not fully feature-separable, limited to those built with two-feature products.  For certain problems, the expectation over the reference population may still yield cancellation of the curl surface integral between different references even when there is non-linearity that is not feature-separable, but it is easy to see that level of symmetry is likely rare.

\subsection{Markov chain paths}
In applying game theory-based attribution algorithms to modern machine learning problems, the fact that the path used yields a particular axiomatic uniqueness that is irrelevant to the problem at hand is almost never given consideration.  This indicates something important is overlooked in such an arbitrary path choice.  There are alternative possibilities for defining the path in a more meaningful way based on auxiliary information that will satisfy the desired axioms.  In the following discussion we note that other path choices can easily be justified.  Indeed, as long as the path choice has a symmetric portion and a portion solely defined by feature identifying information then it will be symmetry-preserving.

For some applications, a possible way to define a meaningful path is through transition probabilities, $P(\vec{x}^\prime\vert \vec{x})$, in what is known as a stochastic game and was first outlined in \cite{shapley1953stochastic}.  The implication being that attributions are sought under the constraint of possible ways to transition between states.  That is not the only possible intention behind assigning attribution, but thinking about it highlights the influence of path choice in forming attribution.

For a single path attribution technique, a logical path choice for assigning attribution with respect to transition is that which maximizes the Markov chain probability,
\begin{equation}
\prod_{i=0}^{m-1} P(\vec{x}_{i+1}\vert \vec{x}_i),
\end{equation}
with the constraint that the path passes through the desired beginning and ending points and does not involve zero delta features ($\Delta x_{i_{ab}}=0$).  This path choice satisfies all of the desired axioms, but does not guarantee invariance under nonlinear transformation of features.

In a sense, the choice of a straight-line path assumes that the transition probability is equal in all directions and takes the most likely path as the dominant contributor to attribution.  Similarly, the \emph{serial cost sharing} path assumes that transition is limited to equal parts in all features until their final value is reached.  Alternatively, interventional Shapley can be thought of as assuming that the only transitions possible are those that change a single feature at a time completely from its starting to ending value.

Indeed, that is the notion of an atomic game, which is quite useful when it applies.  After all, there is no need to consider what score would result when adding half of a player to a sporting team.  Half of a player is a nonsensical proposition, but most machine learning applications do not meet the notion of an atomic game.  Variables are continuously-valued.

Remarkably, for atomic games that are modeled by decision tree, the straight-line path used in GIG will recover the correct Shapley values because its attribution at discontinuities is defined precisely as the Shapley value.  This makes for a more unified approach between piecewise constant, continuous, and piecewise continuous scoring functions.

\section{Attribution Experiments}

To further demonstrate the impact of path choice in assigning attribution we use a toy two-dimensional problem involving binary classification where model uncertainty plays a role.  While it is possible to synthesize a model that will yield significant differences in the resulting attribution, something more realistic provides greater potential insight as to what to expect from real machine learning problems.  For the class probability we take a simple binomial distribution with variation in feature space defined by a hyperbolic tangent function which introduces a transition from one class to another,
\begin{equation}
P(Class \hspace{1pt} 2 \vert \vec{x})=\frac{1}{2}+\frac{1}{2}tanh \left( \frac{1}{\delta}(x_2-f(x_1)) \right)
\end{equation}
with $\delta=\frac{1}{8}$ and
\begin{equation}
f(x_1)= -\left(4(x_1-1/2)\right)^{11}-x_1+1.
\end{equation}
The resulting probability field can be seen in Figure \ref{fig_prob}.  This choice of distribution is made solely to highlight the difference between the path choices of interventional Shapley and Aumann-Shapley.  If the distribution of features is uniform and independent then leveraging any of the path integration methods and taking an expectation over the entire population as a reference will yield similar results with subtle differences.  The following subsections exposes the cause behind those subtle differences by focusing on a particular reference population and a few explanation points belonging to $Class \hspace{2pt}2$ and falling along the line $x_2=x_1$.

\begin{figure}
\centering
\includegraphics[height=.5\textwidth]{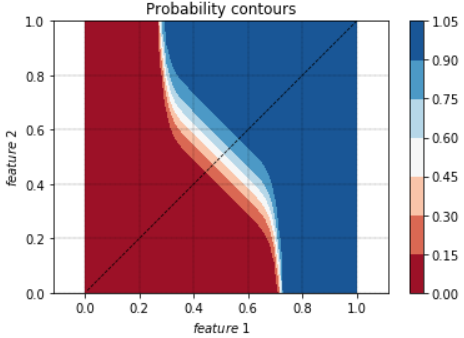}
\caption{A toy probability distribution defined by a hyperbolic tanget with a polynomial center line.}
\label{fig_prob}
\end{figure}

\subsection{Double Gaussian distributed features}
Consider a distribution for $P(\vec{x})$ defined by a two-variable double Gaussian distribution,
\begin{equation}
P(\vec{x}) = P_{gauss_0}(\vec{x})+P_{gauss_1}(\vec{x})
\end{equation}
Each of the two Gaussians contributes to the total $P(\vec{x})$ according to
\begin{equation}
P_{gauss}(\vec{x}) \propto e^{-\frac{1}{2}(\vec{x}-\vec{\mu})^T\Sigma^{-1}(\vec{x}-\vec{\mu})}.
\end{equation}
For $P_{gauss_0}$ let $\vec{\mu}=(0,0)$ and for $P_{gauss_1}$ let $\vec{\mu}=(1,1)$.  The variance $\Sigma$ is taken to be the same for each Gaussian and such that it is diagonal when rotated by $45^{^\circ}$.  In the rotated coordinates $\sigma_{45^{^\circ}}=\frac{2}{5}$ and $\sigma_{135^{^\circ}}=\frac{1}{5}$.  This distribution shows a high degree of correlation between features.  We take a baseline sampling from each Gaussian of $300$ points for a total of $600$ points from the two combined.  These points can be seen in Figure \ref{fig_exp1} along with the resulting contours of an ExtraTreesClassifier \citep{geurts2006extremely,scikit-learn} trained from them.  Without the wider range of feature values, the contours are highly symmetric and for points along the center line we expect equal attribution for both features will result with respect to the reference population defined by the two sampled Gaussians.  To verify this we take a selection of $20$ points along the center line belonging to $Class \hspace{2pt} 2$ (see Figure \ref{fig_exp1_subset}) and calculate the attributions (see Figure \ref{fig_exp1_attr}).  The ratio, seen to the right in Figure \ref{fig_exp1_attr}, demonstrates nearly equal attribution for all points explored.  Closer to the transition regions, the attributions decrease, in accordance with the ExtraTreesClassifier score.

\begin{figure}
\centering
\includegraphics[height=.33\textwidth]{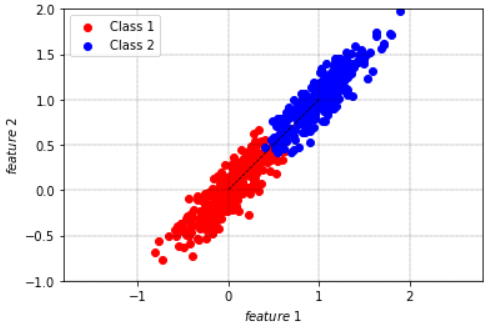}\hfill
\includegraphics[height=.34\textwidth]{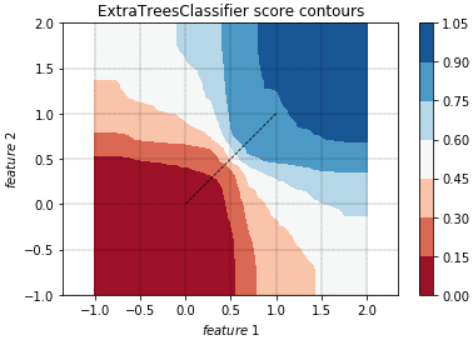}
\caption{\emph{(Left)} The $600$ points of the double Gaussain distribution.  The class smoothly transitions from one to another along the center line.  \emph{(Right)}  The resulting prediction contours of a ExtraTreesClassifier trained from the double Gaussain distribution samples shows a high degree of symmetry about the center line.}
\label{fig_exp1}
\end{figure}

\begin{figure}
\centering
\includegraphics[height=.5\textwidth]{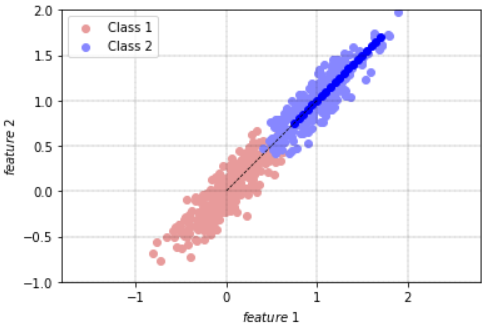}
\caption{A subset of the data is taken for comparison.  The subset is along the center-line of the $Class \hspace{2pt} 2$ population and is seen in darker blue over the general double Gaussian generated population in a lighter shade.}
\label{fig_exp1_subset}
\end{figure}

\begin{figure}
\centering
\includegraphics[height=.32\textwidth]{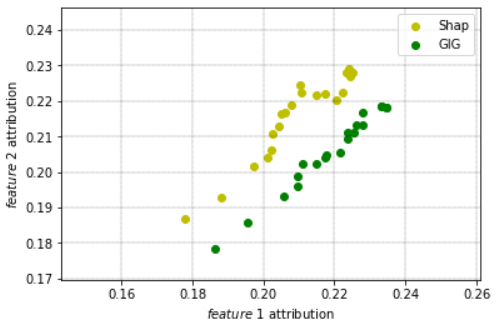}\hfill
\includegraphics[height=.32\textwidth]{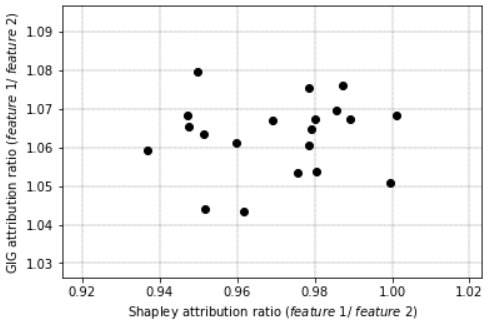}
\caption{\emph{(Left)} The attributions assigned by interventional Shapley and GIG for the double Gaussian distribution of Figure \ref{fig_exp1}. \emph{(Right)}  The ratio of the attributions shows roughly unity for the subset of points along the center line from Figure \ref{fig_exp1_subset}.}
\label{fig_exp1_attr}
\end{figure}


\subsection{Double Gaussian distributed features with a uniform distributed background}
Next we add a \emph{background} uniform distribution of points to the two-variable double Gaussian distribution,
\begin{equation}
P(\vec{x}) = P_{gauss_0}(\vec{x})+P_{gauss_1}(\vec{x})+P_{uniform}(\vec{x}).
\end{equation}
We use this \emph{background} uniform distribution as a means to illuminate the true nature of the probability distribution in the wider region.  We assume that the double Gaussian distribution population always far exceeds that from the uniform distribution.  To avoid requiring a large sampling from the double Gaussian distribution we therefore simply ignore the uniform distribution population when assigning attributions.  In other words all attributions are based solely on the original $600$ samples from the double Gaussian distribution.  This simply highlights the contributions of that particular population, an important thing to understand when considering any wider set of reference population distributions used for estimating attribution expectations.

At first we introduce a mere $25$ points taken from a uniform distribution.  The resulting sample of points and ExtraTreesClassifier score contours are seen in Figure \ref{fig_exp1_25}.  With only a few points covering the background the decision tree provides a conservative estimate of the true probability distribution reflecting the lack of confidence it has due to limited samples.  The attributions from interventional Shapley immediately take note of the changing scoring function due to the orthogonal paths it uses and begins to increase the attribution assigned to $feature$ $1$.  Meanwhile, GIG also increases the credit assigned to the first feature, but to a lesser degree.  We note that at this level of sampling from the uniform distribution interventional Shapley will be much more sensitive to the samples drawn than GIG because the paths of interventional Shapley traverse through the region of uncertainty much more than the straight-line path.  Interestingly, in the previous subsection, the fact that interventional Shapley gave equal attributions was solely a consequence of the piecewise continuous nature of decision trees yielding relatively sane estimates in the absence of any information.  While that extrapolation happened to be reasonable with a decision tree, a higher order continuous modeling technique might not be expected to be quite as lucky, particularly in asymmetric situations.
\begin{figure}
\centering
\includegraphics[height=.33\textwidth]{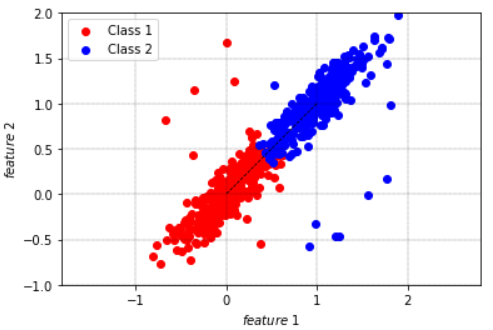}\hfill
\includegraphics[height=.34\textwidth]{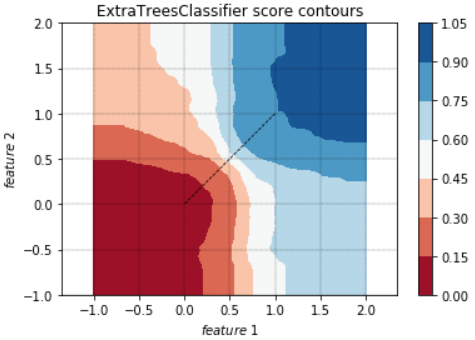}
\caption{\emph{(Left)} The $600$ points of the double Gaussian distribution along with $25$ from a uniform distribution.  The probability field away from the center line is beginning to be revealed.  \emph{(Right)}  The resulting prediction contours of a ExtraTreesClassifier trained from the double Gaussain distribution along with $25$ from a uniform distribution samples.}
\label{fig_exp1_25}
\end{figure}

\begin{figure}
\centering
\includegraphics[height=.33\textwidth]{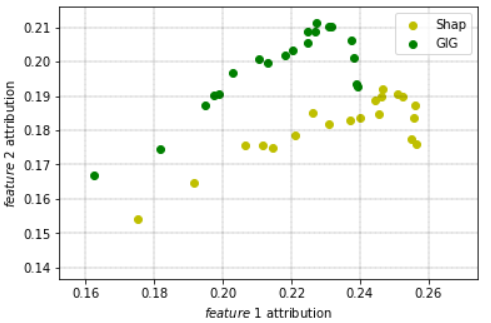}\hfill
\includegraphics[height=.33\textwidth]{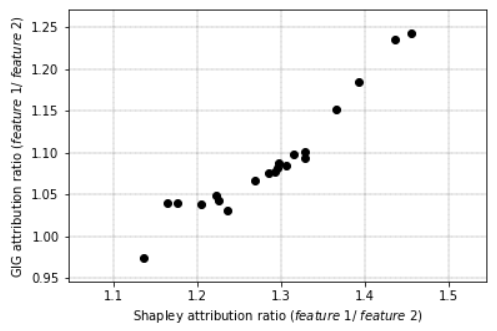}
\caption{\emph{(Left)} The attributions assigned by interventional Shapley and GIG for the double Gaussian distribution along with $25$ uniformly distributed points (see Figure \ref{fig_exp1_25}). \emph{(right)}  The ratio of the attributions shows that interventional Shapley begins to favor $feature$ $1$ for points further away from the transition region.}
\label{fig_exp1_attr_25}
\end{figure}

Continuing, to illuminate the true distribution with a \emph{background} uniform distribution we next consider $100$ points taken from it (see Figure \ref{fig_exp1_100}).  The ExtraTreesClassifier begins to solidify its estimation in the rest of the region.  Interventional Shapley follows suit and increases its attributions for $feature$ $1$.  The maximum ratio of attribution between the features exceeds $2.0$ for interventional Shapley (see Figure \ref{fig_exp1_attr_100}).  Meanwhile the attributions from GIG have not shifted all that much.  Of course this is directly tied to the path choice difference between the two methods.  Interventional Shapley's two paths traverse away from the core double Gaussian distributions into the region where the model is only just building certainty while GIG cuts right across through a region already well supported by samples.

\begin{figure}
\centering
\includegraphics[height=.33\textwidth]{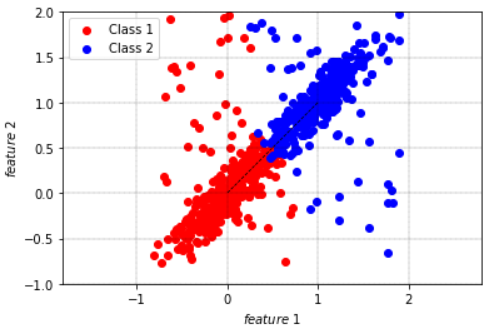}\hfill
\includegraphics[height=.34\textwidth]{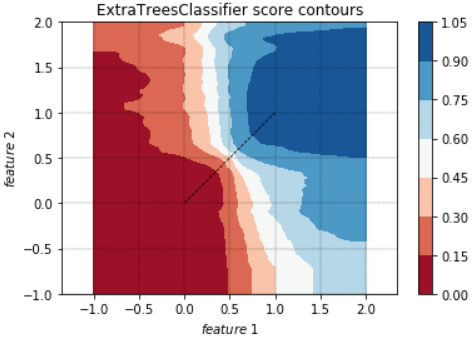}
\caption{\emph{(Left)} The $600$ points of the double Gaussain distribution along with $100$ from a uniform distribution.  The probability field away from the center line is beginning to be revealed.  \emph{(right)}  The resulting prediction contours of a ExtraTreesClassifier trained from the double Gaussain distribution along with $100$ from a uniform distribution samples.}
\label{fig_exp1_100}
\end{figure}

\begin{figure}
\centering
\includegraphics[height=.33\textwidth]{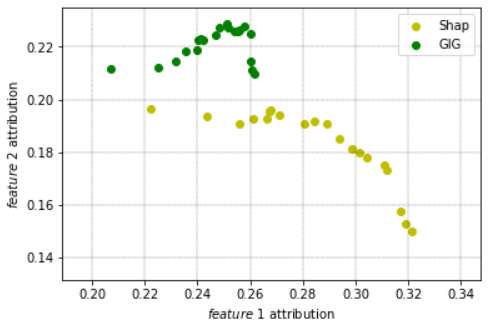}\hfill
\includegraphics[height=.33\textwidth]{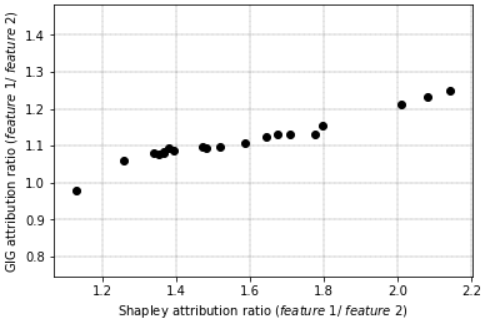}
\caption{\emph{(Left)} The attributions assigned by interventional Shapley and GIG for the double Gaussian distribution along with $100$ uniformly distributed points (see Figure \ref{fig_exp1_100}). \emph{(right)}  The ratio of the attributions shows that interventional Shapley is giving more and more attribution to $feature$ $1$ for points further away from the transition region.}
\label{fig_exp1_attr_100}
\end{figure}

As a last experiment we consider $200$ uniformly distributed points in addition to the $600$ from the double Gaussian distribution (see Figure \ref{fig_exp1_200}).  The decision tree contours are further improving their estimate of the actual classification.  This time a more significant change is seen in the resulting interventional Shapley values (see Figure \ref{fig_exp1_attr_200}).  Again, GIG remains relatively stable in attribution while interventional Shapley now implies that the first feature is as much as $6$ times as important as the second in producing the score.  A truly astonishing difference in attribution that can only be appreciated in light of the actual path choices of the two methods.
\begin{figure}
\centering
\includegraphics[height=.33\textwidth]{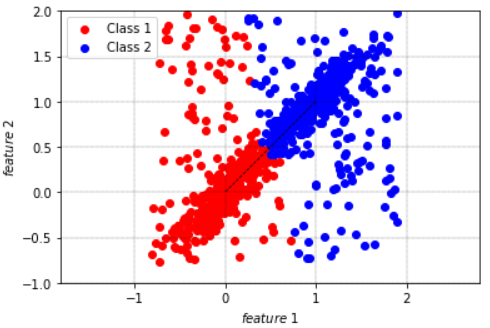}\hfill
\includegraphics[height=.34\textwidth]{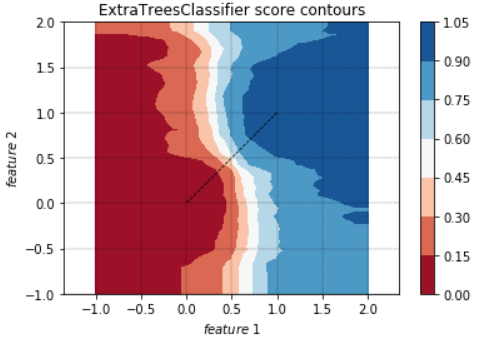}
\caption{\emph{(Left)} The $600$ points of the double Gaussain distribution along with $200$ from a uniform distribution.  The probability field away from the center line is beginning to be revealed.  \emph{(right)}  The resulting prediction contours of a ExtraTreesClassifier trained from the double Gaussain distribution along with $200$ from a uniform distribution samples.}
\label{fig_exp1_200}
\end{figure}

\begin{figure}
\centering
\includegraphics[height=.33\textwidth]{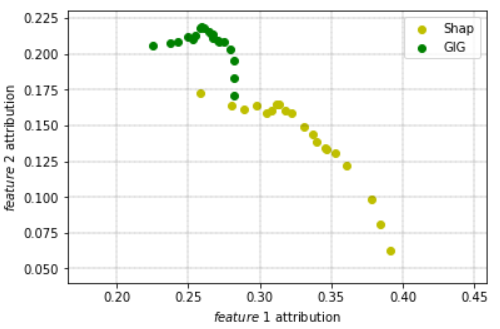}\hfill
\includegraphics[height=.33\textwidth]{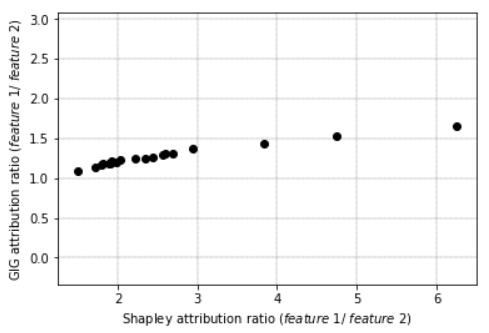}
\caption{\emph{(Left)} The attributions assigned by interventional Shapley and GIG for the double Gaussian distribution along with $200$ uniformly distributed points (see Figure \ref{fig_exp1_200}). \emph{(right)}  The ratio of the attributions shows that interventional Shapley is again giving more and more attribution to $feature$ $1$ for points further away from the transition region.}
\label{fig_exp1_attr_200}
\end{figure}

The ultimate message here is not that one algorithm is better than the other, but merely that explanations must be understood in light of the path choices of the respective methods.  If an atomic transition between reference and explanation point are meaningful to the machine learning problem at hand then interventional Shapley is the right choice.  Alternatively, if a straight-line transition has meaning, then Aumann-Shapley becomes the right choice.  This raises the question of what to do if the right choice is not clear?

In most situations a straight-line path between points will remain closer to the manifold of points than the paths of interventional Shapley.  Interventional Shapley attributions will be more sensitive to outliers.  For some problems, a decision tree could yield a sensible explanation in the regions through which the $n!$ paths of interventional Shapley pass.  But the fact that interventional Shapley often traverses regions where the model prediction is highly uncertain is not reassuring.  While in some cases it may not matter, in many cases, the straight-line path integration techniques will be more robust under uncertainty.

\section{Conclusions}
The path choice is fundamental to the meaning of explanations generated by game theory-based attribution methods.  Claims of uniqueness among such attribution methods stem from the nature of the path.  Interventional Shapley's $n!$ paths yield a uniqueness that is tied to the notion of an atomic game, where infinitesimal transitions are not relevant.  The straight-line path of Aumann-Shapley inevitably has unique qualities of its own.  For most modern machine learning applications the axioms that make these two methods unique are not relevant.  Despite this, such methods remain a fixture of modern model explainability because, in spite of their shortcomings, they still produce useful attributions in a computationally practical manner.  This, however, should not relieve them of more careful scrutiny and attention to where they might fail.

In particular we have shown that the path choice of interventional Shapley can generate attributions that are sensitive to model uncertainty in situations where the data does not conform to the characteristics of an atomic game.  This is due to the fact that models are built from a training data set that often has blind spots, and in these areas of missing data, the model's prediction can be uncertain or even arbitrary.  In such cases, at least some of the $n!$ paths of interventional Shapley will be much further away from the training data manifold than a simple straight-line path. The straight-line path will therefore create more stable attributions compared with interventional Shapley.

Alternatively, in data that does reflect an atomic game, one would think that Aumann-Shapley's straight line path could be sensitive to uncertainty due to the data missingness inherent in atomic games, however, as we have outlined, Generalized Integrated Gradients (GIG) \citep{merrill2019generalized} still generates Shapley values in such situations when they are modeled by decision trees.  That's because GIG, which extends the Aumann-Shapley method of straight-line path integration for infinetsimal games (first applied in \cite{sundararajan2017axiomatic}) to models with discontinuities orthogonal to input features, uses the Shapley value equation when assigning attribution at discontinuities and in continuous regions it is equivalent to IG.  This makes GIG ideal for applications involving mixed models and discrete scoring functions (e.g., forests of trees, compositions of trees and neural networks).



\bibliography{citations.bib}

\begin{thebibliography}{21}
\providecommand{\natexlab}[1]{#1}
\providecommand{\url}[1]{\texttt{#1}}
\expandafter\ifx\csname urlstyle\endcsname\relax
  \providecommand{\doi}[1]{doi: #1}\else
  \providecommand{\doi}{doi: \begingroup \urlstyle{rm}\Url}\fi

\bibitem[Aas et~al.(2019)Aas, Jullum, and L{\o}land]{aas2019explaining}
Kjersti Aas, Martin Jullum, and Anders L{\o}land.
\newblock Explaining individual predictions when features are dependent: More
  accurate approximations to shapley values.
\newblock \emph{arXiv preprint arXiv:1903.10464}, 2019.

\bibitem[Aumann and Shapley(1974)]{aumann-shapley}
Robert~J. Aumann and Lloyd~S. Shapley.
\newblock \emph{Values of Non-Atomic Games}.
\newblock Princeton University Press, 1974.

\bibitem[Castro et~al.(2009)Castro, G{\'o}mez, and
  Tejada]{castro2009polynomial}
Javier Castro, Daniel G{\'o}mez, and Juan Tejada.
\newblock Polynomial calculation of the shapley value based on sampling.
\newblock \emph{Computers \& Operations Research}, 36\penalty0 (5):\penalty0
  1726--1730, 2009.

\bibitem[Chen et~al.(2019)Chen, Lundberg, and Lee]{chen2019explaining}
Hugh Chen, Scott Lundberg, and Su-In Lee.
\newblock Explaining models by propagating shapley values of local components.
\newblock \emph{arXiv preprint arXiv:1911.11888}, 2019.

\bibitem[Friedman et~al.(1999)Friedman, Moulin, et~al.]{friedman1999three}
Eric Friedman, Herve Moulin, et~al.
\newblock Three methods to share joint costs or surplus.
\newblock \emph{Journal of economic Theory}, 87\penalty0 (2):\penalty0
  275--312, 1999.

\bibitem[Friedman(2004)]{FriedmanPaths}
Eric~J. Friedman.
\newblock Paths and consistency in additive cost sharing.
\newblock \emph{International Journal of Game Theory}, 32\penalty0
  (4):\penalty0 5--32, 2004.

\bibitem[Geurts et~al.(2006)Geurts, Ernst, and Wehenkel]{geurts2006extremely}
Pierre Geurts, Damien Ernst, and Louis Wehenkel.
\newblock Extremely randomized trees.
\newblock \emph{Machine learning}, 63\penalty0 (1):\penalty0 3--42, 2006.

\bibitem[Janzing et~al.(2019)Janzing, Minorics, and
  Bl{\"o}baum]{janzing2019feature}
Dominik Janzing, Lenon Minorics, and Patrick Bl{\"o}baum.
\newblock Feature relevance quantification in explainable {AI}: A causality
  problem.
\newblock \emph{arXiv preprint arXiv:1910.13413}, 2019.

\bibitem[Kononenko et~al.(2010)]{kononenko2010efficient}
Igor Kononenko et~al.
\newblock An efficient explanation of individual classifications using game
  theory.
\newblock \emph{Journal of Machine Learning Research}, 11\penalty0
  (Jan):\penalty0 1--18, 2010.

\bibitem[Lundberg and Lee(2017{\natexlab{a}})]{lundberg2017consistent}
Scott~M Lundberg and Su-In Lee.
\newblock Consistent feature attribution for tree ensembles.
\newblock \emph{arXiv preprint arXiv:1706.06060}, 2017{\natexlab{a}}.

\bibitem[Lundberg and Lee(2017{\natexlab{b}})]{lundberg2017unified}
Scott~M Lundberg and Su-In Lee.
\newblock A unified approach to interpreting model predictions.
\newblock In I.~Guyon, U.~V. Luxburg, S.~Bengio, H.~Wallach, R.~Fergus,
  S.~Vishwanathan, and R.~Garnett, editors, \emph{Advances in Neural
  Information Processing Systems 30}, pages 4765--4774. Curran Associates,
  Inc., 2017{\natexlab{b}}.
\newblock URL
  \url{http://papers.nips.cc/paper/7062-a-unified-approach-to-interpreting-model-predictions.pdf}.

\bibitem[Lundberg et~al.(2018)Lundberg, Erion, and Lee]{lundberg2018consistent}
Scott~M Lundberg, Gabriel~G Erion, and Su-In Lee.
\newblock Consistent individualized feature attribution for tree ensembles.
\newblock \emph{arXiv preprint arXiv:1802.03888}, 2018.

\bibitem[Maleki et~al.(2013)Maleki, Tran-Thanh, Hines, Rahwan, and
  Rogers]{maleki2013bounding}
Sasan Maleki, Long Tran-Thanh, Greg Hines, Talal Rahwan, and Alex Rogers.
\newblock Bounding the estimation error of sampling-based shapley value
  approximation.
\newblock \emph{arXiv preprint arXiv:1306.4265}, 2013.

\bibitem[Merrick and Taly(2019)]{merrick2019explanation}
Luke Merrick and Ankur Taly.
\newblock The explanation game: Explaining machine learning models with
  cooperative game theory.
\newblock \emph{arXiv preprint arXiv:1909.08128}, 2019.

\bibitem[Merrill et~al.(2019)Merrill, Ward, Kamkar, Budzik, and
  Merrill]{merrill2019generalized}
John Merrill, Geoff Ward, Sean Kamkar, Jay Budzik, and Douglas Merrill.
\newblock Generalized integrated gradients: A practical method for explaining
  diverse ensembles.
\newblock \emph{arXiv preprint arXiv:1909.01869}, 2019.

\bibitem[Pearl(2009)]{pearl2009causality}
Judea Pearl.
\newblock \emph{Causality}.
\newblock Cambridge university press, 2009.

\bibitem[Pedregosa et~al.(2011)Pedregosa, Varoquaux, Gramfort, Michel, Thirion,
  Grisel, Blondel, Prettenhofer, Weiss, Dubourg, Vanderplas, Passos,
  Cournapeau, Brucher, Perrot, and Duchesnay]{scikit-learn}
F.~Pedregosa, G.~Varoquaux, A.~Gramfort, V.~Michel, B.~Thirion, O.~Grisel,
  M.~Blondel, P.~Prettenhofer, R.~Weiss, V.~Dubourg, J.~Vanderplas, A.~Passos,
  D.~Cournapeau, M.~Brucher, M.~Perrot, and E.~Duchesnay.
\newblock Scikit-learn: Machine learning in {P}ython.
\newblock \emph{Journal of Machine Learning Research}, 12:\penalty0 2825--2830,
  2011.

\bibitem[Shapley(1953{\natexlab{a}})]{shapley}
Lloyd~S. Shapley.
\newblock A value for n-person games.
\newblock In H.~W. Kuhn and A.~W. Tucker, editors, \emph{Contributions to the
  Theory of Games}, volume~28 of \emph{Annals of Mathematical Studies}, pages
  307--317. Princeton University Press, 1953{\natexlab{a}}.

\bibitem[Shapley(1953{\natexlab{b}})]{shapley1953stochastic}
Lloyd~S Shapley.
\newblock Stochastic games.
\newblock \emph{Proceedings of the national academy of sciences}, 39\penalty0
  (10):\penalty0 1095--1100, 1953{\natexlab{b}}.

\bibitem[Sundararajan and Najmi(2019)]{sundararajan2019many}
Mukund Sundararajan and Amir Najmi.
\newblock The many shapley values for model explanation.
\newblock \emph{arXiv preprint arXiv:1908.08474}, 2019.

\bibitem[Sundararajan et~al.(2017)Sundararajan, Taly, and
  Yan]{sundararajan2017axiomatic}
Mukund Sundararajan, Ankur Taly, and Qiqi Yan.
\newblock Axiomatic attribution for deep networks.
\newblock In \emph{Proceedings of the 34th International Conference on Machine
  Learning-Volume 70}, pages 3319--3328. JMLR. org, 2017.

\end{thebibliography}

\end{document}